\newif\ifisit
\isitfalse

\ifisit
\documentclass[conference,letterpaper]{IEEEtran}
\IEEEoverridecommandlockouts
\else

\documentclass{article}
\fi

\usepackage[utf8]{inputenc}
\usepackage{amsmath}

\usepackage{amssymb}

 \usepackage{algcompatible} %
\usepackage{algorithm}     %
 \usepackage{algorithmicx}  %
 \usepackage{algpseudocode}
\usepackage{lipsum}
\usepackage{algpseudocode}
\usepackage{caption}
\usepackage{xspace}

\usepackage[normalem]{ulem}

\usepackage{tikz}
\usetikzlibrary{calc,positioning, math, tikzmark}

\usepackage{pgfplots}

\usepackage{amsthm}
\usepackage[colorlinks, citecolor=blue]{hyperref}

\ifisit
\else
\usepackage[left=3cm,right=3cm,bottom=4cm]{geometry}
\fi

\usepackage{bbm}
\usepackage{bm}

\usepackage{comment}

\usepackage{wrapfig}
\newif\ifmylinenumbers
\mylinenumbersfalse

\ifisit
\else
\usepackage{breqn}

\fi

\ifmylinenumbers

\usepackage[mathlines]{lineno}
\linenumbers

\newcommand*\patchAmsMathEnvironmentForLineno[1]{%
   \expandafter\let\csname old#1\expandafter\endcsname\csname #1\endcsname
   \expandafter\let\csname oldend#1\expandafter\endcsname\csname end#1\endcsname
   \renewenvironment{#1}%
      {\linenomath\csname old#1\endcsname}%
      {\csname oldend#1\endcsname\endlinenomath}}%
\newcommand*\patchBothAmsMathEnvironmentsForLineno[1]{%
   \patchAmsMathEnvironmentForLineno{#1}%
   \patchAmsMathEnvironmentForLineno{#1*}}%

\AtBeginDocument{%
  \patchBothAmsMathEnvironmentsForLineno{equation}%
\patchBothAmsMathEnvironmentsForLineno{align}%
\patchBothAmsMathEnvironmentsForLineno{flalign}%
\patchBothAmsMathEnvironmentsForLineno{alignat}%
\patchBothAmsMathEnvironmentsForLineno{gather}%
\patchBothAmsMathEnvironmentsForLineno{multline}%
}

\newcounter{EditBlock}
\newcounter{CondEditBlock}

\newcommand{\editstart}{\renewcommand{\linenumberfont}{\normalfont\sffamily\footnotesize\bf\color{blue}}\stepcounter{EditBlock}\linelabel{L:editstart_\theEditBlock}}
\newcommand{\editfinish}{\renewcommand{\linenumberfont}{\normalfont\sffamily\tiny\color{black}}\linelabel{L:editfinish_\theEditBlock}}

\else
\newcommand{\editstart}{}
\newcommand{\editfinish}{}
\fi
\newcommand{\ev}[1]{\mathbb{E} \left [ #1 \right ] }
\newcommand{\evwrt}[2]{\mathbb{E}_{#1} \left [ #2 \right ] }
\newcommand{\pr}[1]{\mathbb{P} \left ( #1 \right ) }

\newcommand{\snorm}[1]{\Vert #1 \Vert}

\newcommand{\one}[1]{\mathbbm{1} \left [ #1 \right ]}

\newcommand{\reals}{\mathbb{R}}

\newtheorem{lem}{Lemma}
\newtheorem*{lem*}{Lemma}
\newtheorem{thm}{Theorem}

\newtheorem{rem}{Remark}

\newtheorem{assumption}{Assumption}

\newcommand{\tsum}{\mathsf{TSum}}

\newcommand{\vx}{\vec{x}}

\newcommand{\vu}{\vec{u}}

\newcommand{\vw}{\vec{w}}

\newcommand{\vxp}{\vec{x}'}

\let\vec\bm

\newcommand{\xp}{x'}

\newcommand{\mN}{\mathcal{N}}

\newcommand{\mI}{\mathcal{I}}
\newcommand{\mB}{\mathcal{B}}
\newcommand{\mC}{\mathcal{C}}

\newcommand{\roberr}{\mathcal{L}}
\newcommand{\roberrd}{\mathcal{L}^{(d)}} 
\newcommand{\optroberrd}{{\mathcal{L}^*}^{(d)}}
\newcommand{\optroberr}{{\mathcal{L}^*}}

\algnewcommand\algorithmicinput{\textbf{Input:}}     %
\algnewcommand\INPUT{\item[\algorithmicinput]}       %
\algnewcommand\algorithmicoutput{\textbf{Output:}}   %
\algnewcommand\OUTPUT{\item[\algorithmicoutput]}     %
\algrenewcommand\algorithmicrequire{\textbf{Input:}} %
\algrenewcommand\algorithmicensure{\textbf{Output:}} %

\newcommand{\nd}{q} 
\newcommand{\ndp}{q'} 
\newcommand{\ndd}{q^{(d)}}

\newcommand{\xd}{x^{(d)}}
\newcommand{\vxd}{\vec{x}^{(d)}}
\newcommand{\xpd}{{x'}^{(d)}}
\newcommand{\vxpd}{{\vxp}^{(d)}}
\newcommand{\txd}{\widetilde{x}^{(d)}}
\newcommand{\tx}{\widetilde{x}}

\newcommand{\txpd}{{\widetilde{x}}^{\prime (d)}}
\newcommand{\vtxpd}{{\widetilde{\vec{x}}}^{\prime (d)}}
\newcommand{\vtxd}{\widetilde{\vec{x}}^{(d)}}

\newcommand{\mCd}{{\mathcal{C}}^{(d)}}

\newcommand{\convdist}[1]{\xrightarrow[#1]{\text{dist}}}
\newcommand{\convprob}[1]{\xrightarrow[#1]{\text{prob}}}

\newcommand{\vwd}{\vec{w}^{(d)}}
\newcommand{\wdd}{w^{(d)}}

\newcommand{\tp}{\widetilde{p}}

\newcommand{\dtv}{d_{\text{TV}}}

\allowdisplaybreaks

\begin{document}

\title{Binary Classification Under $\ell_0$ Attacks for General Noise Distribution}

\author{Payam Delgosha\thanks{Department of Computer Science, University of
    Illinois at Urbana-Champaign, IL, \texttt{delgosha@illinois.edu}}
  \qquad Hamed Hassani\thanks{Department of Electrical and Systems Engineering,
    University of Pennsylvania, Philadelphia, PA, \texttt{hassani@seas.upenn.edu}}
  \qquad Ramtin Pedarsani\thanks{Department of Electrical and Computer
    Engineering, University of California, Santa Barbara, Santa Barbara, CA, \texttt{ramtin@ece.ucsb.edu}}}

\maketitle

\begin{abstract}
Adversarial examples have recently drawn considerable attention in
the field of machine learning due to the fact that small perturbations in the
data can result in major performance degradation. This
phenomenon is usually modeled by  a malicious adversary that can
apply  perturbations to
the data in a constrained fashion, such as being bounded in a certain norm. In
this paper, we study this problem when the adversary is constrained by the
$\ell_0$ norm; i.e., it can perturb a certain number of coordinates in the input,
but has no limit on how much it can perturb those coordinates. Due to the
combinatorial nature of this setting, we need to go beyond the standard techniques
in robust machine learning to address this problem. We consider a binary
classification scenario where $d$ noisy data samples of the true label are provided to
us after adversarial perturbations.
We introduce a  classification
method which employs a nonlinear component called truncation, and show
in an asymptotic scenario, as long as the adversary is restricted to perturb no
more than $\sqrt{d}$ data samples, we can almost achieve the optimal classification error
in the absence of the adversary, i.e.\ we can completely neutralize adversary's
effect. Surprisingly, we observe a phase transition in the sense that using a
converse argument, we show that if the adversary can perturb more than
$\sqrt{d}$ coordinates, no classifier can do better than a random guess.
\end{abstract}


\section{Introduction}

It is well-known that machine learning models are susceptible to adversarial attacks that can cause classification error. These attacks are typically in the form of a small norm-bounded perturbation to the input data that are carefully designed to incur misclassification  -- e.g. they can be form of an additive $\ell_p$-bounded perturbation for some $p\geq 0$ \cite{biggio2013evasion,szegedy,goodfellow2014explaining,carlini2017,Madry_ICLR}.

There is an extensive body of prior work studying adversarial machine learning, most of which have focused on $\ell_2$ and $\ell_\infty$ attacks \cite{carlini2018,marzi, bhattacharjee2021consistent, bhattacharjee2020sample}. 
To train models that are more robust against such attacks, adversarial training
is the state-of-the-art defense method. However, the success of the current
adversarial training methods is mainly based on empirical  evaluations \cite{Madry_ICLR}.  It is therefore imperative to study the fundamental limits of robust machine learning under different classification settings and attack models.

In this paper, we focus on the important case of $\ell_0$-bounded attacks that
has been less investigated so far. In such attacks,  given an $\ell_0$ budget $k$, an adversary can
change $k$ entries of the input vector in an arbitrary fashion -- i.e. the adversarial perturbations belong to the so-called $\ell_0$ ball of radius $k$. In contrast with
$\ell_p$-balls \textcolor{black}{($p \geq 1$)}, the  $\ell_0$-ball
is non-convex and non-smooth.
Moreover, the $\ell_0$ ball
contains inherent discrete (combinatorial) structures that can be exploited by
both the learner and the adversary. As a result, the $\ell_0$-adversarial
setting bears various challenges that are absent in common $\ell_p$-adversarial settings. In thus regard, it has recently been shown that any piece-wise linear classifier, e.g. a
feed-forward deep neural  network with ReLu activations, completely fails in the
$\ell_0$ setting \cite{shamir2019simple}.

Perturbing only a few components of the data or signal has many real-world  applications
{\color{black}including} natural
language processing~\cite{jin2019bert},  malware
detection~\cite{grosse2016adversarial}, and physical attacks in object  detection~\cite{li2019adversarial}. 
There have been several prior works on $\ell_0$-adversarial attacks including white-box attacks that are gradient-based,
e.g.~\cite{carlini2017,papernot2016limitations,modas2019sparsefool}, and black-box attacks 
based on zeroth-order optimization,
e.g.~\cite{schott2018towards,croce2020sparse}. Defense strategies against
$\ell_0$-bounded attacks have also been proposed, e.g. defenses based on
randomized ablation~\cite{levine2020robustness} and defensive
distillation~\cite{papernot2016distillation}. None of the above works have studied the fundamental limits of the $\ell_0$-adversarial setting theoretically. In our prior work, we have studied the $\ell_0$-adversarial setting for the case of Gaussian mixture model \cite{delgosha2021robust}. In this paper, we generalize our results to the case of binary classification with general noise distribution.

The goal of this paper is to characterize
the optimal classifier and the corresponding robust classification error as a function of the adversary's budget
$k$. More precisely, we focus on the binary classification setting with general but i.i.d.\ noise distributions, where the input is generated according to the following model: $x_i = y \mu + z_i$, where $y \in \{ -1 ,1 \}$ is the true label, $z_i$ is a zero-mean i.i.d.\ random noise process, and $\mu$ is its mean vector. We seek to find the robust classification error of the optimal classifier in this setting. In other words, we would like to study ``how robust'' we can design a
classifier given a certain budget for an $\ell_0$ adversary. Specifically, we consider the asymptotic regime that the dimension of the input gets large, and ask the following fundamental question: What is the maximum adversary's budget for which the optimal error in the absence of an adversary (standard error) can still be achieved and how does this limit scale with the input's dimension? 

The main contributions of the paper to answer the above questions are as follows. 

\begin{itemize}
    \item We prove an achievability result by introducing a classifier and
      characterizing its performance.
        Our proposed classification method finds the likelihood of
        each data sample, and applies \emph{truncation} by removing a few of the largest
        and a few of the smallest values. This truncation phase effectively removes the
        ``outliers'' present in the input due to adversarial modification. We
        have shown in a previous work \cite{delgosha2021robust} that truncation is effective to
        robustify against $\ell_0$ attacks in a
        Gaussian         mixture  setting. The present work shows the effectiveness of
        this method in a much broader setting for general noise distributions.
\item We prove a converse result by finding a lower bound  on the  optimal
robust error, and show that the two bounds asymptotically match as the  dimension $d
\rightarrow \infty$, hence our proposed classification method is  optimally
robust against such adversarial attacks.
  The key idea behind the converse proof is to use techniques from
  the optimal transport theory and studying the asymptotic behavior of the
  maximal coupling between the data distribution under the two labels $+1$ and
  $-1$. We use such a coupling to design a strategy for the adversary by making
the distribution ``look almost the same'' under the two labels, hence removing
the information about the true label.
\item Surprisingly, we observe a phase
transition for the optimal robust error in terms of the adversary's budget.
Roughly speaking,  we observe that if the adversary's budget is below
$\sqrt{d}$, we can asymptotically achieve the optimal standard error
which corresponds to the case where there is no adversary, while if the adversary's budget is above $\sqrt{d}$, no classifier
can do better than a random guess. In other words, we can totally  compensate
for the presence of the adversary as  long as its budget is below $\sqrt{d}$ and
achieve a performance \emph{as if there were no adversary}. On the other hand, above this
threshold $\sqrt{d}$, the adversary can perturb the data in such a way that the
information about the true label is  lost and hence no classifier can do better
than a random guess. Consequently,  \emph{there is no trade-off between
  robustness and accuracy in this setting}.
    
\end{itemize}

We close
 this section by introducing some notation. 
We denote the set of integers $\{1, \dots, n\}$ by $[n]$. $\bar{\Phi}(x):=
\frac{1}{\sqrt{2\pi}}\int_x^\infty \exp(-t^2/2)dt$ denotes the complementary CDF
of a standard normal distribution. $\mN(\mu, \sigma^2)$ denotes a 
real-valued normal distribution with mean $\mu$ and variance $\sigma^2$.
$\convdist{}$ and $\convprob{}$ denote convergence in distribution and convergence
in probability, respectively. $X \sim p(.)$ means that the random variable $X$
has distribution $p(.)$. We use the boldface notation for vectors in the
Euclidean space, e.g.\ $\vx \in \reals^d$.



\section{Problem Formulation}
\label{sec:problem-formulation}

We consider the binary classification setting where the true label is $Y \sim
\text{Unif}\{\pm 1\}$ and conditioned on a realization $y$, $d$ independent real-valued data samples $\xd_1,
\dots, \xd_d$ are generated such that  $\xd_i = y\mu_d + z_i$.
Here, $\mu_d \in \reals$ is the conditional expectation of $\xd_i$ given $y=1$ and
$z_1, \dots, z_d$ are i.i.d.\ samples of a zero-mean real-valued noise
distribution which has a 
density $\nd(.)$. We  consider a high-dimensional setting where the dimension $d
\rightarrow \infty$, and $\mu_d$ can depend on the data dimension $d$. However,
we assume that 
the noise density $\nd(.)$ is fixed and known. Note that since the $\ell_0$ norm is invariant
  under scalar multiplication, we can arbitrarily normalize the quantities, and
   this assumption is made without loss of generality. 
   We denote the vector of the input data samples by $\vxd = (\xd_i: i \in [d])$.
Throughout this paper, the superscript $(d)$ emphasizes the dependence on the
dimension $d$. However, we may drop it from the notations whenever the dimension
is clear from the context.
A classifier is a measurable function $\mC: \vx \mapsto \{\pm 1\}$ which predicts the true
label from the input $\vx$. We consider the 0-1 loss $\ell(\mC; \vx, y) :=
\one{\mC(\vx) \neq y}$ as a metric for discrepancy between the prediction of the
classifier on the input $\vx$ and the true label $y$.

We assume that an adversary is allowed to perturb the input $\vx$ within the
$\ell_0$ ball of radius $k$: 
\begin{equation*}
  \mB_0(\vxd, k) := \{\vxpd \in \reals^d: \snorm{\vxd - \vxpd}_0 \leq k\},
\end{equation*}
where $\snorm{\vxd}_0 := \sum_{i=1}^d \one{\xd_i \neq 0}$. Effectively, the
adversary can change at most $k$ data samples. The parameter $k$ is called the
\emph{adversary's budget}. 
Similar to the above, whenever the dimension $d$ is clear from the context, we
may denote the adversary's perturbed data samples as $\vxp = (\xp_i: i \in [d])$.
In this setting, the \emph{robust classification error} (or
\emph{robust error} for short) associated to a classifier $\mC$ is defined to be 
\begin{equation}
  \label{eq:robust-error-def}
  \roberrd_{\mu_d, \nd}(\mC, k) := \ev{\max_{\vxp \in \mB_0(\vx, k)} \ell(\mC; \vxp, y)},
\end{equation}
where the expectation is taken with respect to the above mentioned distribution
parametrized by $d, \mu_d$, and $\nd$.
 The \emph{optimal robust classification error} (or \emph{optimal robust error} for
short) is defined by optimizing the robust error over all possible (measurable) classifiers:
\begin{equation}
  \label{eq:optimal-robust-error-def}
  \optroberrd_{\mu_d, \nd}(k) := \inf_{\mC} \roberrd_{\mu_d, \nd}(\mC, k).
\end{equation}
In words, $\optroberrd_{\mu_d, \nd}(k)$ is the minimum error that any
classifier can achieve in the presence of an adversary with an $\ell_0$ budget
$k$. In other words, no classifier can obtain a robust error smaller than
$\optroberrd_{\mu_d, \nd}(k)$ in this setting. 
Whenever the problem parameters are clear from the context,
we may drop them from the notation and write $\roberrd(\mC, k)$ or $\roberr(\mC,
k)$, and $\optroberrd(k)$ or $\optroberr(k)$.

In the absence of the adversary, or equivalently  when $k=0$, $\optroberr(0)$ reduces to
the \emph{optimal standard error}, which is optimal Bayes error of estimating
$Y$ upon observing the noisy samples $x_1, \dots, x_d$. In order to fix the
baseline, specifically to have a meaningful asymptotic discussion as $d
\rightarrow \infty$, we assume that $\mu_d$ is such that the optimal standard
error $\optroberrd_{\mu_d, \nd}(0)$ remains constant as $d \rightarrow
\infty$. As we will see later (see Theorem~\ref{thm:std-err-sqrt-d} in Section~\ref{sec:asymp-std-err}),  this is achieved
when $\mu_d = c/\sqrt{d}$ for some $c >0$. Motivated by this, we study the setting where
$\mu_d = c / \sqrt{d}$ for some constant $c > 0$ throughout this paper. 
When $\mu_d = c / \sqrt{d}$ when $c < 0$, similar results still hold after
substituting $c$ with $|c|$.

\vspace{2mm}

\section{Main Results}
\label{sec:main-results}

In order to prove our main results, we need the following assumptions on the
noise distribution $\nd(.)$. We will show later (see Section~\ref{sec:exp-family}) that all of these
assumptions are satisfied for a large class of distributions, including the exponential
family of distributions with polynomial exponents, e.g.\ the normal
distribution. 

\begin{assumption}
  \label{assumption:Fisher}
  We have $\nd(z) > 0$ for all $z \in \reals$,  $ \nd(.)$ is three times
  continuously differentiable, and
  \begin{equation}
\label{eq:int-qp-qpp-zero}
    \int_{-\infty}^\infty \nd'(z) dz = \int_{-\infty}^\infty \nd''(z) dz = 0,
  \end{equation}
 where $\nd'(.)$ and $\nd''(.)$ denote the first and second derivatives of $\nd(.)$.
  Furthermore, the location family of distributions
  \begin{equation}
    \label{eq:q-z-t-def}
    \nd(z;\theta):= \nd(z-\theta),
  \end{equation}
  parameterized by $\theta \in \reals$ has well-defined and finite Fisher information
  $\{\mI_{\nd}(\theta)\}_{\theta \in \reals}$.
\end{assumption}

The Fisher information of the parametric family of distributions $\nd(z;\theta)$
where $z,\theta \in \reals$ is defined to be
\begin{equation*}
  \mI_{\nd}(\theta):= \int \left( \frac{\partial}{\partial \theta} \log \nd(z;\theta) \right)^2 \nd(z;\theta)dz.
\end{equation*}
See, for instance, \cite{lehmann2006theory} for more details.
Since $\nd(z;\theta) = \nd(z - \theta)$ is a location family, it turns out that
$\mI_{\nd}(\theta)$ is independent of $\theta$. The common value, which we
denote by $\mI_{\nd}$ by an abuse of notation, is given by
\begin{equation}
\label{eq:Iq-integral-qp2-a}
  \mI_{\nd} := \int_{-\infty}^\infty \frac{(\nd'(z))^2}{\nd(z)} dz.
\end{equation}

\begin{assumption}
  \label{assumption:zeta-d3}
  There exists $\zeta > 0$ such that
  \begin{equation}
    \label{eq:zeta-d3}
    \evwrt{Z \sim \nd(.)} { \sup_{t \in [Z-\zeta, Z+\zeta]} \left |\frac{d^3}{dt^3} \log \nd(t) \right | } < \infty.
  \end{equation}
\end{assumption}

\begin{assumption}
  \label{assumption:zeta-e-d2}
  There exist $\zeta > 0$ such that
  \begin{equation}
    \label{eq:zeta-e-d2}
    \evwrt{Z \sim \nd(.)}{\sup_{t \in [Z-\zeta, Z+\zeta]} \left|\frac{d^2}{dt^2} \log \nd(t) \right|^{2}} < \infty.
  \end{equation}
\end{assumption}

\begin{assumption}
  \label{assumption:max-log-q-log}
There exist constants $\gamma >0 $ and  $C_4 > 0$ such that
  \begin{equation*}
     \lim_{d \rightarrow \infty}\pr{\max_{1 \leq i \leq d} \left | \frac{d}{d z} \log \nd(Z_i) \right | >  C_4 (\log d)^{\gamma}} = 0,
   \end{equation*}
   where $Z_i$ are i.i.d.\ with distribution $\nd(.)$.
\end{assumption}

The following theorem formalizes the phase transition we discussed previously,
i.e.\ if adversary's budget is orderwise below $\sqrt{d}$, we can totally compensate
for its presence, while if adversary's budget is orderwise above $\sqrt{d}$, no classifier
can do better than a random guess. As we discusses previously, we assume that
$\mu_d = c/ \sqrt{d}$ for a constant $c > 0$ to ensure that the standard error
is asymptotically constant (see Theorem~\ref{thm:std-err-sqrt-d} in Section~\ref{sec:asymp-std-err}).

\begin{thm}
  \label{thm:main-phase-transition}
  Assume that $\mu_d = c / \sqrt{d}$ for some constant $c > 0$, and the
  assumptions~\ref{assumption:Fisher}-\ref{assumption:max-log-q-log} are
  satisfied for the noise density $\nd(.)$. Then, if $k_d$ is a sequence of
  adversary's $\ell_0$ budget, then we have
  \begin{enumerate}
  \item If $\limsup_{d \rightarrow \infty} \log_d k_d < 1/2$, there exists a
    sequence of classifiers $\mCd_{k_d}$ such that 
    \begin{equation*}
      \limsup_{d \rightarrow \infty} \roberrd_{\mu_d, \nd}(\mCd_{k_d}, k_d) - \optroberrd_{\mu_d, \nd}(0) = 0.
    \end{equation*}
    In other words, the excess risk of this sequence of classifiers as compared
    to the optimal standard error (when there is no adversary) converges to
    zero.
  \item If $\liminf_{d \rightarrow \infty} \log_d k_d > 1/2$, we have
    \begin{equation*}
      \liminf_{d \rightarrow \infty} \optroberrd_{\mu_d, \nd}(k_d) \geq 1/2.
    \end{equation*}
    In other words, no classifier can asymptotically do better than a random
    guess. 
  \end{enumerate}
\end{thm}

\ifisit
\else
The proof of this result, which is given in Appendix~\ref{app:main-phase-transition-proof}, essentially follows
from Theorems~\ref{thm:upper-bound} and~\ref{thm:converse}.
\fi
More precisely, in
Section~\ref{sec:achievability}, we prove an achievability result by introducing
a sequence of robust classifiers in the sub-$\sqrt{d}$ regime (first part of the theorem), while in
Section~\ref{sec:converse-lower-bound}, we prove a converse result by
introducing a strategy for the adversary in the super-$\sqrt{d}$ regime which perturbs the data in such a way
that the information about the true label is asymptotically removed (second part of
the theorem). See~\cite{isitfull} for a complete proof of Theorem~\ref{thm:main-phase-transition}.

\subsection{Asymptotic Standadrd Error}
\label{sec:asymp-std-err}

Recall that in the absence of the adversary, or equivalently when adversary's
budget $k$ is zero, the optimal robust error $\optroberrd_{\mu_d, \nd}(0)$
reduces to the optimal Bayes error of estimating $Y$ upon observing the noisy
samples $x_1, \dots, x_d$. With an abuse of notation, we write
$\optroberrd_{\mu_d, \nd}$ (or $\optroberr$ for short) for this optimal Bayes
error. Our goal in this section is to find the appropriate scaling of $\mu_d$
with $d$ such that $\optroberrd_{\mu_d, \nd}$ converges to a constant as
$d\rightarrow \infty$.

In order to characterize $\optroberr$, note that since there is no adversary,
and the prior on $Y$ is uniform, the optimal Bayes classifier is the maximum
likelihood estimator that computes the likelihood
\begin{equation}
  \label{eq:likelihood-txd-def}
  \sum_{i=1}^d \txd_i \qquad \text{where} \qquad \txd_i := \log \frac{q(\xd_i - \mu_d)}{q(\xd_i + \mu_d)},
\end{equation}
and returns the estimate $\hat{y}$ of $y$ as
\begin{equation}
\label{eq:no-adv-ML-estimator}
  \hat{y} =
  \begin{cases}
    1 & \sum_{i=1}^d \txd_i > 0 \\
    -1 & \text{otherwise}.
  \end{cases}
\end{equation}

The following Theorem~\ref{thm:std-err-sqrt-d} shows that if $\mu_d=c/\sqrt{d}$, then the optimal Bayes error converges to a constant.
\ifisit
\else
The proof of Theorem~\ref{thm:std-err-sqrt-d} is given in Appendix~\ref{app:std-err-asymptotic}.
\fi
\begin{thm}
  \label{thm:std-err-sqrt-d}
  Assume that assumptions~\ref{assumption:Fisher} and
  \ref{assumption:zeta-d3} are satisfied for the noise density $\nd(.)$. Then,
  if $\mu_d = \frac{c}{\sqrt{d}}$ for some constant $c > 0$, we have
  \begin{equation*}
    \lim_{d \rightarrow \infty} \optroberrd_{\mu_d, \nd} = \bar{\Phi}(c \sqrt{\mI_{\nd}}).
  \end{equation*}
Furthermore, in this case, as $d\rightarrow \infty$, conditioned on $Y=+1$, the log likelihood $\sum_{i=1}^d \txd_i$ converges in
distribution to a normal  $\mN(2 c^2 \mI_{\nd}, 4c^2 \mI_{\nd})$ where $\mI_{\nd}$
was defined in~\eqref{eq:Iq-integral-qp2-a} and is the Fisher information  associated to the location
family defined in~(\ref{eq:q-z-t-def}). Moreover, conditioned on $Y = -1$,
$\sum_{i=1}^d \txd_i$ converges in
distribution to a normal  $\mN(-2 c^2 \mI_{\nd}, 4c^2 \mI_{\nd})$.
\end{thm}

\begin{rem}
  As we will see in Appendix~\ref{app:std-err-asymptotic}, if $c < 0$, we need
  to replace $\bar{\Phi}(c \sqrt{\mI_{\nd}})$ by $\bar{\Phi}(|c|
  \sqrt{\mI_{\nd}})$ in the above theorem.
\end{rem}

\subsection{Achievability: Upper Bound on the Optimal Robust Error}
\label{sec:achievability}


In this section, we introduce a classifier and study its robustness against
$\ell_0$ adversarial perturbations. Recall that if $k$ is the adversary's
budget, the input to the classifier is $\vxp = (\xp_1, \dots, \xp_d)$ which is
different from the original sequence $x_1, \dots, x_d$ in at most $k$
coordinates. Recall from Section~\ref{sec:asymp-std-err} that in the absence of
the adversary, the optimal Bayes classifier is the maximum likelihood estimator 
based on $\sum_{i=1}^d \tx_i$, as was defined in~\eqref{eq:likelihood-txd-def}.
Motivated by this, we define
\begin{equation}
  \label{eq:txp-def}
  \txpd_i := \log \frac{\nd(\xpd_i - \mu_d)}{\nd(\xpd_i + \mu_d)}.
\end{equation}
Note that if $\vtxpd$ denotes the vector $(\txpd_i: i \in [d])$, since
$\snorm{\vxpd - \vxd}_0 \leq k$, we have
\begin{equation}
  \label{eq:txpd-txd-l0-k}
  \snorm{\vtxpd - \vtxd}_0 \leq k.
\end{equation}
We define the truncated classifier $\mCd_k$ as follows.
Given a vector $\vu =(u_i: i \in [d]) \in \reals^d$ and an integer $k \geq 0$, we define the truncated
summation $\tsum_k(\vu)$ to be the summation of coordinates in $\vu$ except for the
top and bottom $k$ coordinates. More precisely, let $\vec{s} = (s_i: i \in [d]) =
\text{sort}(\vu)$ be obtained by sorting the coordinates of $\vu$ in descending
order. We then define
\begin{equation}
  \label{eq:TSum-def}
  \tsum_k(\vu) := \sum_{i=k+1}^{d-k} s_i.
\end{equation}
When $k=0$, this indeed reduces to the normal summation. Motivated by~\eqref{eq:txpd-txd-l0-k}, 
we replace $\sum_{i=1}^d \txd_i$ with its \emph{robustified} version 
$\tsum_k(\sum_{i=1}^d \txpd_i)$ and define 
\begin{equation}
  \label{eq:cdk-def}
  \mCd_k(\vxpd) :=
  \begin{cases}
    + 1 & \tsum_k(\vtxpd) > 0 \\
    - 1 & \text{otherwise}.
  \end{cases}
\end{equation}
This method essentially removes the ``outliers'' introduced by the adversary into
the data.

The following theorem shows that this classifier is asymptotically robust
against adversarial attacks with $\ell_0$ budget of at most $\sqrt{d}$. A
matching lower bound is provided in  Section~\ref{sec:converse-lower-bound}.
\ifisit
\else
The proof of Theorem~\ref{thm:upper-bound} below is Given in Appendix~\ref{app:upper-bound-proof}.
\fi

\begin{thm}
  \label{thm:upper-bound}
  Assume that
  Assumptions~\ref{assumption:Fisher}-\ref{assumption:max-log-q-log} are
  satisfied for the noise density $\nd(.)$, and $\mu_d = c / \sqrt{d}$ for some
  $c>0$. Then if $k_d$ is a sequence of adversary's budgets so that $k_d <
  d^{\frac{1}{2} - \epsilon}$ for some $\epsilon >0$, then we have
  \begin{equation}
    \label{eq:trun-class-asymp-error}
    \limsup_{d \rightarrow \infty} \roberrd_{\mu_d, \nd}(\mCd_{k_d}, k_d) \leq \bar{\Phi}(c\sqrt{\mI_{\nd}}).
  \end{equation}
  In particular, we have
  \begin{equation}
    \label{eq:trunc-class-excess-zero}
    \limsup_{d \rightarrow \infty} \roberrd_{\mu_d, \nd}(\mCd_{k_d}, k_d) - \optroberrd_{\mu_d, \nd} = 0.
  \end{equation}
\end{thm}
Note that $\optroberrd_{\mu_d, \nd}$, as was defined in
Section~\ref{sec:asymp-std-err} above,  is the optimal Bayes error in an ideal
scenario when there is
no adversary, and $\roberrd_{\mu_d, \nd}(\mCd_{k_d}, k_d) -
\optroberrd_{\mu_d, \nd}$ is the excess error  of our truncated classifier
with respect to this ideal scenario. 
In fact, \eqref{eq:trunc-class-excess-zero} implies that our truncated
classifier is asymptotically optimal in   the specified regime of adversary's
budget. The truncated classifier manages to compensate for the presence
of the adversary, and performs as if there is no adversary.


\begin{rem}
  As we will see in Appendix~\ref{app:upper-bound-proof}, if $c < 0$, we need
  to replace $\bar{\Phi}(c \sqrt{\mI_{\nd}})$ by $\bar{\Phi}(|c|
  \sqrt{\mI_{\nd}})$ in the above theorem.
\end{rem}

\subsection{Converse: Lower Bound on the Optimal Robust Error}
\label{sec:converse-lower-bound}

In this section, we provide a lower bound on the optimal robust error.
\ifisit
\else
In
Section~\ref{sec:achievability}, we observed that roughly speaking, if adversary's
budget is below $\sqrt{d}$, we can asymptotically compensate for its  effect  and recover the Bayes optimal error, as if no adversary is present. In
this section, we show that, roughly speaking, if adversary's budget is above
$\sqrt{d}$, no classifier can asymptotically do better than a random guess, resulting in a
robust error of $1/2$.
\fi
We do this by introducing an attack strategy for the
adversary. In this strategy, the adversary with a sufficiently large budget,
 perturbs the input data in such a way that  all the
information about the true label $Y$ is lost, resulting in a perturbed data which has a
vanishing correlation  with the true label.
\ifisit
\else
The proof of
Theorem~\ref{thm:converse} below is given in Appendix~\ref{app:converse-proof}.
\fi

\begin{thm}
  \label{thm:converse}
  Assume that Assumptions~\ref{assumption:Fisher} and \ref{assumption:zeta-e-d2}
  are satisfied for the noise density $\nd(.)$, and
  $\mu_d = c / \sqrt{d}$ for some $c> 0$. Then, if $k_d$ is a sequence of
  adversary's budgets so that $k_d > d^{1/2 + \epsilon}$ for some $\epsilon >
  0$,  we have
  \begin{equation*}
    \liminf_{d \rightarrow \infty} \optroberrd_{\mu_d, \nd}(k_d) \geq 1/2.
  \end{equation*}
\end{thm}

\subsection{Exponential Family of Distributions}
\label{sec:exp-family}

\ifisit
\else
In this section, we show that the
Assumptions~\ref{assumption:Fisher}-\ref{assumption:max-log-q-log} are all
satisfied for a large class of distributions, namely the exponential family of noise
distributions of the form
\begin{equation}
\label{eq:q-exp-poly}
  \nd(z) = \frac{\exp(\psi(z))}{A},
\end{equation}
where
\begin{equation*}
  \psi(z) = -a_{2n} z^{2n} + a_{2n-1} z^{2n-1} + \dots a_1 z + a_0,
\end{equation*}
is a polynomial in $z$ with even degree $2n > 0$  such that $a_{2n} > 0$. Here, $A :=
\int_{-\infty}^\infty \psi(z) dz$ is the normalizing constant.
Note that since $\psi(.)$ has an even degree with a negative leading
coefficient, we have  $A <\infty$.

\ifisit
Proof ideas of Theorem~\ref{thm:exp-poly} below are given in
Appendix~\ref{app:outline-exp-poly}. See~\cite{isitfull} for a complete proof.
\fi

\begin{thm}
\label{thm:exp-poly}
  Assumptions~\ref{assumption:Fisher}-~\ref{assumption:max-log-q-log} are all
  satisfied for the density $\nd(.)$ of the form~\eqref{eq:q-exp-poly}.
\end{thm}
\fi


\section{Conclusion}
\label{sec:conclusion}

We studied the binary classification problem in the presence of an adversary
 constrained by the $\ell_0$ norm. We introduced a robust classification
method which employs truncation on the log likelihood. We showed
that this classification method can  asymptotically compensate for the presence
of the adversary as long as adversary's budget is orderwise below $\sqrt{d}$. Moreover,
we showed a phase transition through a converse argument in the sense that  no classifier can asymptotically do better than a
random guess if adversary's budget is orderwise  above $\sqrt{d}$. 


\newcommand{\etalchar}[1]{$^{#1}$}

\ifisit

\else
\appendix
\section{Proof of Theorem~\ref{thm:std-err-sqrt-d}}
\label{app:std-err-asymptotic}

\editstart

\begin{proof}[Proof of Theorem~\ref{thm:std-err-sqrt-d}]
To simplify the discussion and to avoid considering multiple cases, it turns out
that it is more convenient to assume that the constant $c$ can be negative.
Therefore, for the rest of the proof, we assume that $\mu_d = c / \sqrt{d}$
where $c \in \reals$ and $c \neq 0$.
Note that even for negative $c$, the maximum likelihood estimator
in~\eqref{eq:no-adv-ML-estimator} is still the optimal Bayes estimator.
Therefore
\begin{equation}
    \label{eq:no-adv-mls-pr-tx-negative}
\begin{aligned}
  \optroberrd_{\mu_d, \nd} &=\frac{1}{2} \pr{\sum_{i=1}^d \tx_i \leq 0 \Big | Y = + 1 } + \frac{1}{2} \pr{\sum_{i=1}^d \tx_i > 0 \Big | Y = -1 } \\
  &= \frac{1}{2} \pr{\sum_{i=1}^d \log \frac{\nd(z_i)}{\nd(z_i + 2 \mu_d)} \leq 0 } + \frac{1}{2} \pr{\sum_{i=1}^d \log \frac{\nd(z_i - 2 \mu_d)}{\nd(z_i)} > 0}.
\end{aligned}
\end{equation}
We focus on the first term.
  Using $\mu_d = c / \sqrt{d}$, we may write
  \begin{equation}
    \label{eq:sum-log-ll-mult-div-mud}
    \sum_{i=1}^d \log \frac{\nd(z_i)}{\nd(z_i + 2 \mu_d)} = \frac{c}{\sqrt{d}} \sum_{i=1}^d \frac{1}{\mu_d} \log \frac{\nd(z_i)}{\nd(z_i + 2 \mu_d)}.
  \end{equation}
From Assumption~\ref{assumption:Fisher}, we know that $\nd(.)$ is positive
everywhere and three times continuously differentiable, hence $\log \nd(.)$ is three
times continuously differentiable. Therefore, writing the Taylor expansion, we
get
\begin{equation*}
  \log \nd(z_i + 2 \mu_d) = \log \nd(z_i) + 2 \mu_d \frac{d}{dz} \log \nd(z_i) + \frac{4\mu_d^2}{2} \frac{d^2}{dz^2} \log \nd(z_i) + \frac{8\mu_d^3}{6} \frac{d^3}{dz^3} \log \nd(z_i + \epsilon_i), 
\end{equation*}
where $|\epsilon_i| < 2 |\mu_d| = 2|c| / \sqrt{d}$. Note that $\epsilon_i$ is  random and only
depends on $z_i$. Substituting this
into~(\ref{eq:sum-log-ll-mult-div-mud}), we get
\begin{equation}
  \label{eq:T1-T2-T3}
\begin{aligned}
  \sum_{i=1}^d \log \frac{\nd(z_i)}{\nd(z_i + 2 \mu_d)} &= \underbrace{\frac{-2c}{\sqrt{d}} \sum_{i=1}^d \frac{d}{dz} \log \nd(z_i)}_{=:T_1} \\
  & \underbrace{\frac{-2c\mu_d}{\sqrt{d}} \sum_{i=1}^d \frac{d^2}{dz^2} \log \nd(z_i)}_{=:T_2} \\
  & \underbrace{\frac{-4c \mu_d^2}{3\sqrt{d}} \sum_{i=1}^d \frac{d^3}{dz^3} \log \nd(z_i + \epsilon_i)}_{=:T_3}.
\end{aligned}
\end{equation}
We now study each of the three terms individually.

\underline{$T_1$:} Denoting $\frac{d}{dz} \nd(z)$ by $\ndp(z)$, we have
\begin{equation}
\label{eq:ev-ddz-log-q-0}
  \evwrt{z \sim \nd(.)}{\frac{d}{dz} \log \nd(z)} = \evwrt{z \sim \nd(.)}{\frac{\ndp(z)}{\nd(z)}} = \int_{-\infty}^\infty \frac{\ndp(z)}{\nd(z)} \nd(z) dz = \int_{-\infty}^\infty \ndp(z) dz = 0,
\end{equation}
where last equality uses Assumption~\ref{assumption:Fisher}. On the other hand,
\begin{equation*}
  \frac{d}{d\theta} \log \nd(z - \theta) \Big \vert_{\theta = 0}= - \frac{\ndp(z-\theta)}{\nd(z-\theta)} \Big \vert_{\theta = 0} = -\frac{\ndp(z)}{\nd(z)} = - \frac{d}{dz} \log \nd(z).
\end{equation*}
Therefore,
\begin{align*}
  \evwrt{z \sim \nd(.)}{\left( \frac{d}{dz} \log \nd(z) \right)^2} &= \evwrt{z \sim \nd(.)}{\left( \frac{d}{d\theta} \log \nd(z-\theta) \right)^2} \Bigg \vert_{\theta = 0} \\
                                                                   &= \evwrt{z \sim \nd(.)}{\left( \frac{d}{d\theta} \log \nd(z;\theta) \right)^2} \Bigg \vert_{\theta = 0} \\
  &= \mI_{\nd},
\end{align*}
where $\mI_{\nd}$ is the Fisher information  associated to the
location family of distributions $\nd(z;\theta)$ defined
in~\eqref{eq:q-z-t-def}.
Note that Assumption~\ref{assumption:Fisher} ensures that $\mI_{\nd}$ is
well-defined and finite. 
Therefore, combining this with~\eqref{eq:ev-ddz-log-q-0} and using the central limit
theorem, we realize that
\begin{equation*}
  \frac{1}{\sqrt{d}} \sum_{i=1}^d \frac{d}{dz} \log \nd(z_i) \convdist{d\rightarrow \infty} \mN(0,\mI_{\nd}).
\end{equation*}
Consequently
\begin{equation}
  \label{eq:T1-limit-dist}
  T_1 = \frac{-2c}{\sqrt{d}} \sum_{i=1}^d \frac{d}{dz} \log \nd(z_i) \convdist{d\rightarrow \infty} \mN(0,4c^2 \mI_{\nd}).
\end{equation}

\underline{$T_2$:} Since $\mu_d = c / \sqrt{d}$, we have
\begin{equation}
  \label{eq:T2-equiv-2c2}
  T_2 = \frac{-2c^2}{d} \sum_{i=1}^d \frac{d^2}{dz^2} \log \nd(z_i).
\end{equation}
On the other hand, note that
\begin{equation}
  \label{eq:ev-d2dz2-log-q-1}
  \evwrt{z \sim \nd(.)}{\frac{d^2}{dz^2} \log \nd(z)} = \evwrt{z \sim \nd(.)}{\frac{\partial^2}{\partial \theta^2} \log \nd(z - \theta)} \Bigg|_{\theta =0}.
\end{equation}
Using Assumption~\ref{assumption:Fisher} and \cite[Lemma
5.3]{lehmann2006theory}, we have
\begin{equation*}
  \mI_{\nd} = \mI_{\nd}(0) = \evwrt{z \sim \nd(.)}{\left( \frac{\partial}{\partial \theta} \log \nd(z - \theta) \right)^2} \Bigg|_{\theta = 0} = - \evwrt{z \sim \nd(.)}{\frac{\partial^2}{\partial\theta^2} \log \nd(z - \theta)}\Bigg|_{\theta =0}.
\end{equation*}
Substituting this into~\eqref{eq:ev-d2dz2-log-q-1}, we get
\begin{equation*}
  \evwrt{z \sim \nd(.)}{\frac{d^2}{dz^2} \log \nd(z)} = - \mI_{\nd}.
\end{equation*}
Since $\mI_{\nd} < \infty$ from Assumption~\ref{assumption:Fisher}, using the
law of large numbers in~\eqref{eq:T2-equiv-2c2}, we realize that
\begin{equation}
\label{eq:T2-lim-as}
  \lim_{d \rightarrow \infty} T_2 = 2c^2 \mI_{\nd} \qquad \text{a.s.}.
\end{equation}

\underline{$T_3$:} Since $\mu_d = c / \sqrt{d}$, we
may bound $T_3$ as follows:
\begin{align*}
  |T_3| &= \left| \frac{-4c^3}{3 \sqrt{d}} \frac{1}{d} \sum_{i=1}^d \frac{d^3}{dz^3} \log \nd(z_i + \epsilon_i) \right| \\
        &\leq \frac{4|c|^3}{3 \sqrt{d}} \frac{1}{d} \sum_{i=1}^d \left|  \frac{d^3}{dz^3} \log \nd(z_i + \epsilon_i) \right| \\
  &\leq \frac{4|c|^3}{3 \sqrt{d}} \frac{1}{d} \sum_{i=1}^d \sup_{t \in [z_i-2|\mu_d|, z_i + 2 |\mu_d|]}\left|  \frac{d^3}{dz^3} \log \nd(t) \right|,
\end{align*}
where the last line uses the fact that $\epsilon_i \in (0,2 \mu_d)$. Since
$\mu_d = c / \sqrt{d} \rightarrow 0$ as $d \rightarrow \infty$, we have $2|\mu_d| <
\zeta$ for $d$ large enough,  where $\zeta$ is the constant in
Assumption~\ref{assumption:zeta-d3}. Thereby, for $d$ large enough, we have
\begin{equation}
  \label{eq:T3-bound-2-sup-zeta}
  |T_3| \leq \frac{4|c|^3}{3 \sqrt{d}} \frac{1}{d} \sum_{i=1}^d \sup_{t \in [z_i-\zeta, z_i + \zeta]}\left|  \frac{d^3}{dz^3} \log \nd(t) \right|.
\end{equation}
Using the law of large numbers together with
Assumption~\ref{assumption:zeta-d3}, we have
\begin{equation*}
   \lim_{d \rightarrow \infty} \frac{1}{d} \sum_{i=1}^d \sup_{t \in [z_i-\zeta, z_i + \zeta]}\left|  \frac{d^3}{dz^3} \log \nd(t) \right| =     \evwrt{Z \sim \nd(.)}{\sup_{t \in [Z-\zeta, Z+\zeta]} \left|\frac{d^2}{dt^2} \log \nd(t) \right|^{2}} < \infty \qquad \text{a.s.}.
\end{equation*}
This together with~\eqref{eq:T3-bound-2-sup-zeta} implies that
\begin{equation}
  \label{eq:T3-lim-as}
  \lim_{d \rightarrow \infty} T_3 = 0 \qquad \text{a.s.}.
\end{equation}
Using~\eqref{eq:T1-limit-dist}, \eqref{eq:T2-lim-as}, and \eqref{eq:T3-lim-as}
back into \eqref{eq:T1-T2-T3}, we realize that
\begin{equation}
\label{eq:qz-qz+2mu-normal}
  \sum_{i=1}^d \log \frac{\nd(z_i)}{\nd(z_i + 2 \mu_d)} \convdist{d \rightarrow \infty} \mN(2c^2 \mI_q, 4c^2 \mI_q).
\end{equation}
Substituting $c$ with $-c$ in this result, we get
\begin{equation*}
  \sum_{i=1}^d \log \frac{\nd(z_i)}{\nd(z_i - 2 \mu_d)} \convdist{d \rightarrow \infty} \mN(2c^2 \mI_q, 4c^2 \mI_q),
\end{equation*}
or equivalently
\begin{equation}
  \label{eq:qz-2-mud-1z-normal}
  \sum_{i=1}^d \log \frac{\nd(z_i - 2 \mu_d)}{\nd(z_i)} \convdist{d \rightarrow \infty} \mN(-2c^2 \mI_q, 4c^2 \mI_q).
\end{equation}
Using~\eqref{eq:qz-qz+2mu-normal} and \eqref{eq:qz-2-mud-1z-normal} back
into~\eqref{eq:no-adv-mls-pr-tx-negative}, we get
\begin{align*}
  \lim_{d \rightarrow \infty} \optroberrd_{\mu_d, \nd} &= \frac{1}{2} \pr{\mN(2c^2 \mI_q, 4c^2 \mI_q) \leq 0} + \frac{1}{2} \pr{\mN(-2c^2 \mI_q, 4c^2 \mI_q) > 0} \\
                                                       &= \pr{2|c|\sqrt{\mI_q} \mN(0,1) > 2c^2 \mI_q} \\
                                                       &= \bar{\Phi}(|c| \sqrt{\mI_q}).
\end{align*}
Moreover, since the left hand side of ~\eqref{eq:qz-qz+2mu-normal} is precisely
$\sum_{i=1}^d \txd_i$ when $Y = +1$,  we realize that conditioned on $Y=+1$, the log likelihood $\sum_{i=1}^d \txd_i$ converges in
distribution to a normal  $\mN(2 c^2 \mI_{\nd}, 4c^2 \mI_{\nd})$. Likewise,
\eqref{eq:qz-2-mud-1z-normal} implies that conditioned on $Y = -1$,
$\sum_{i=1}^d \txd_i$ converges in
distribution to a normal  $\mN(-2 c^2 \mI_{\nd}, 4c^2 \mI_{\nd})$.
This completes the proof.
\end{proof}

\editfinish


\section{Proof of Theorem~\ref{thm:upper-bound}}
\label{app:upper-bound-proof}

\editstart

The following lemma will be useful in our analysis.

\begin{lem}[Lemma~1 in~\cite{delgosha2021robust}]
\label{lem:tsum-8k}
Given $\vx, \vxp, \vw \in \reals^d$, for integer $k$ satisfying $\snorm{\vx -
  \vxp}_0 \leq k < d/2$, we have
\begin{equation*}
  |\langle \vw, \vxp \rangle_k - \langle \vw, \vx \rangle| \leq 8k \snorm{\vw \odot \vx}_\infty.
\end{equation*}
In particular, for $\vw$ being the all-one vector, we have
\begin{equation*}
  \left|\tsum_k(\vxp) - \sum_{i=1}^d x_i \right| \leq 8k \snorm{\vx}_\infty.
\end{equation*}
\end{lem}

\begin{proof}[Proof of Theorem~\ref{thm:upper-bound}]
It turns out that in order to simplify the discussion and to avoid considering 
multiple cases, it is more convenient to allow $c$ to be negative. Therefore, in
this proof we assume that $\mu_d  = c / \sqrt{d}$ where $c \in \reals$ and $c
\neq 0$. Note that we still stick to the definition of $\vtxpd$
in~\eqref{eq:txp-def} and $\mCd_k$ in~\eqref{eq:cdk-def}.
We have 
\begin{equation}
  \label{eq:robloss-half-prob-equality}
\begin{aligned}
  \roberrd_{\mu_d, \nd}(\mCd_{k_d}, k_d) &= \ev{\max_{\vxpd \in \mB_0(\vxd, k_d)} \one{\mCd_{k_d}(\vxpd) \neq y}} \\
                                      &= \ev{\one{\exists \vxpd \in \mB_0(\vxd, k_d): \mCd_{k_d}(\vxpd) \neq y}} \\
                                      &= \pr{\exists \vxpd \in \mB_0(\vxd, k_d): \mCd_{k_d}(\vxpd) \neq y} \\
                                      &= \frac{1}{2} \pr{\exists \vxpd \in \mB_0(\vxd, k_d): \tsum_k(\vtxpd)\leq 0 |Y = +1} \\
  & \qquad +\frac{1}{2} \pr{\exists \vxpd \in \mB_0(\vxd, k_d): \tsum_k(\vtxpd)> 0 |Y = -1} 
\end{aligned}
\end{equation}
Note that for $\vxpd \in \mB_0(\vxd,k_d)$, we have $\snorm{\vtxpd - \vtxd}_0 \leq
0$. Therefore, using Lemma~\ref{lem:tsum-8k}, we have
\begin{equation*}
  \left(\sum_{i=1}^d \txd_i\right) - 8k_d \snorm{\vtxd}_\infty \leq \tsum_k(\vtxpd) \leq   \left(\sum_{i=1}^d \txd_i\right) + 8k_d \snorm{\vtxd}_\infty \qquad \forall \vxpd \in \mB_0(\vxd, k_d).
\end{equation*}
Using this in~\eqref{eq:robloss-half-prob-equality}, we get
\begin{equation}
\label{eq:ld-cdk-upper-bound-1}
  \roberrd_{\mu_d, \nd(.)}(\mCd_{k_d}, k_d) \leq \frac{1}{2} \pr{\sum_{i=1}^d \txd_i \leq 8k_d \snorm{\vtxd}_\infty\Big|Y=+1} + \frac{1}{2} \pr{\sum_{i=1}^d \txd_i > -8k_d \snorm{\vtxd}_\infty\Big|Y=-1}
\end{equation}

We study each of the two terms separately.

\underline{Conditioned on $Y=+1$,}  we have
\begin{equation}
  \label{eq:txd-log-mud-mult-div}
\begin{aligned}
  \txd_i &= \log \frac{\nd(z_i)}{\nd(z_i + 2 \mu_d)}  \\
         &= \frac{c}{\sqrt{d}} \frac{1}{\mu_d} \log \frac{\nd(z_i)}{\nd(z_i + 2 \mu_d)}.
\end{aligned}
\end{equation}
Using the Taylor expansion, we get
\begin{equation*}
  \log \nd(z_i + 2\mu_d) = \log \nd(z_i) + 2 \mu_d \frac{d}{dz} \log \nd(z_i) + 2 \mu_d^2 \frac{d^2}{dz^2} \log \nd(z_i + \epsilon_i),
\end{equation*}
where $|\epsilon_i| < 2 |\mu_d| = 2|c| / \sqrt{d}$. Using this in
\eqref{eq:txd-log-mud-mult-div}, since $\mu_d = c / \sqrt{d}$, we get
\begin{equation*}
  \txd_i = \frac{-2c}{\sqrt{d}} \frac{d}{dz} \log \nd(z_i) - \frac{2c^2 }{d} \frac{d^2}{dz^2} \log \nd(z_i + \epsilon_i).
\end{equation*}
Consequently,
\begin{equation}
  \label{eq:norm-infty-tx-T1-T2}
\begin{aligned}
  \snorm{\vtxd}_\infty &\leq \underbrace{\frac{2|c|}{\sqrt{d}}\max_{1 \leq i \leq d} \left| \frac{d}{dz} \log \nd(z_i) \right|}_{=:T_1} \\
  &\quad + \underbrace{\frac{2c^2}{d} \max_{1 \leq i \leq d} \sup_{t \in [z_i-2|\mu_d|, z_i + 2|\mu_d|]} \left| \frac{d^2}{dt^2} \log \nd(t) \right|}_{=:T_2}.
\end{aligned}
\end{equation}

\underline{For $T_1$}, note that using
Assumption~\ref{assumption:max-log-q-log}, there are constants $\gamma > 0$ and
$C_4 > 0$ such that 
\begin{equation*}
  \lim_{d \rightarrow \infty} \pr{T_1 > \frac{2c C_4 (\log d)^\gamma}{\sqrt{d}}} = 0.
\end{equation*}
This in particular implies that, since $k_d \leq d^{\frac{1}{2} - \epsilon}$, we
have
\begin{equation}
  \label{eq:k-T1-conv-prob-zero}
  k_d T_1 \convprob{d\rightarrow \infty} 0.
\end{equation}

\underline{For $T_2$,} let $d$ be large enough so that with the constant $\zeta$
in Assumption~\ref{assumption:zeta-e-d2}, we have  $2 |\mu_d| = 2|c| / \sqrt{d} <
\zeta$. For such $d$, we may write
\begin{align*}
  T_2 &\leq \frac{2c^2}{d} \max_{1 \leq i \leq d} \sup_{t \in [z_i-\zeta, z_i + \zeta]} \left| \frac{d^2}{dt^2} \log \nd(t) \right| \\
         &= \frac{2c^2}{d} \left(  \max_{1 \leq i \leq d} \sup_{t \in [z_i-\zeta, z_i + \zeta]} \left| \frac{d^2}{dt^2} \log \nd(t) \right|^2\right)^{1/2} \\
         &= \frac{2c^2}{\sqrt{d}} \left( \frac{1}{d} \max_{1 \leq i \leq d} \sup_{t \in [z_i-\zeta, z_i + \zeta]} \left| \frac{d^2}{dt^2} \log \nd(t) \right|^2\right)^{1/2} \\
  &\leq \frac{2c^2}{\sqrt{d}} \left( \frac{1}{d} \sum_{i=1}^d \sup_{t \in [z_i-\zeta, z_i + \zeta]} \left| \frac{d^2}{dt^2} \log \nd(t) \right|^2\right)^{1/2}.
\end{align*}
Note that from Assumption~\ref{assumption:zeta-e-d2}, we have
\begin{equation*}
  \lim_{d \rightarrow \infty} \frac{1}{d} \sum_{i=1}^d \sup_{t \in [z_i-\zeta, z_i + \zeta]} \left| \frac{d^2}{dt^2} \log \nd(t) \right|^2 =     \evwrt{Z \sim \nd(.)}{\sup_{t \in [Z-\zeta, Z+\zeta]} \left|\frac{d^2}{dt^2} \log \nd(t) \right|^{2}} < \infty \qquad \qquad \text{a.s.}.
\end{equation*}
Thereby, since $k_d < d^{1/2 - \epsilon}$, we have
\begin{equation}
  \label{eq:lim-kd-T2-d-zero-as}
  \lim_{d \rightarrow \infty} k_d T_2 = 0 \qquad \qquad \text{a.s.}.
\end{equation}
Combining this with \eqref{eq:k-T1-conv-prob-zero} and substituting
into~\eqref{eq:norm-infty-tx-T1-T2}, we realize that conditioned on $Y=+1$,
$k_d \snorm{\vtxd}_\infty$ converges to zero in probability as $d\rightarrow
\infty$. On the other hand, from Theorem~\ref{thm:std-err-sqrt-d}, we know that conditioned
on $Y=+1$, $\sum_{i=1}^d \txd_i$ converges in distribution to a normal $\mN(2c^2
\mI_{\nd}, 4c^2 \mI_{\nd})$. Consequently, we have
\begin{equation}
\label{eq:lim-sum-tx-tx-inf-normal-1}
\begin{aligned}
  \lim_{d \rightarrow \infty} \pr{\sum_{i=1}^d \txd_i \leq 8k_d \snorm{\vtxd}_\infty\Big|Y=+1} &= \lim_{d \rightarrow \infty} \pr{\sum_{i=1}^d \log \frac{\nd(z_i )}{\nd(z_i)+ 2 \mu_d} \leq 8 k_d \max_{1 \leq i \leq d} \left| \log \frac{\nd(z_i )}{\nd(z_i) + 2\mu_d} \right|} \\
  &= \pr{\mN(2c^2
    \mI_{\nd}, 4c^2 \mI_{\nd}) \leq 0} \\
  &= \bar{\Phi}(|c|\sqrt{\mI_{\nd}}).
\end{aligned}
\end{equation}

\underline{Conditioned on $Y=-1$}, we have
\begin{equation*}
  \txd_i = \log \frac{\nd(z_i - 2\mu_d)}{\nd(z_i)}.
\end{equation*}
Therefore,
\begin{align*}
  \pr{\sum_{i=1}^d \txd_i > -8k_d \snorm{\vtxd}_\infty\Big|Y=-1} &= \pr{\sum_{i=1}^d \log \frac{\nd(z_i - 2 \mu_d)}{\nd(z_i)} > -8 k_d \max_{1 \leq i \leq d} \left| \log \frac{\nd(z_i - 2\mu_d)}{\nd(z_i)} \right|} \\
  &= \pr{\sum_{i=1}^d \log \frac{\nd(z_i )}{\nd(z_i)- 2 \mu_d} < 8 k_d \max_{1 \leq i \leq d} \left| \log \frac{\nd(z_i )}{\nd(z_i) - 2\mu_d} \right|}
\end{align*}
Comparing this with~\eqref{eq:lim-sum-tx-tx-inf-normal-1}, we realize that by
replacing $c$ with $-c$ in the above discussion for $Y=+1$, we have
\begin{align*}
  \lim_{d\rightarrow \infty}   \pr{\sum_{i=1}^d \txd_i > -8k_d \snorm{\vtxd}_\infty\Big|Y=-1} &= \lim_{d \rightarrow \infty} \pr{\sum_{i=1}^d \log \frac{\nd(z_i )}{\nd(z_i)- 2 \mu_d} < 8 k_d \max_{1 \leq i \leq d} \left| \log \frac{\nd(z_i )}{\nd(z_i) - 2\mu_d} \right|} \\
  &= \pr{\mN(2c^2
    \mI_{\nd}, 4c^2 \mI_{\nd}) \leq 0} \\
  &= \bar{\Phi}(|c|\sqrt{\mI_{\nd}}).
\end{align*}
Combining this with~\eqref{eq:lim-sum-tx-tx-inf-normal-1} and substituting back in~\eqref{eq:ld-cdk-upper-bound-1}, we
get
\begin{equation*}
  \limsup_{d \rightarrow \infty} \roberrd_{\mu_d, \nd}(\mCd_{k_d}, k_d) \leq \bar{\Phi}(|c|\sqrt{\mI_{\nd}}),
\end{equation*}
which completes the proof.
\end{proof}

\editfinish


\section{Proof of Theorem~\ref{thm:converse}}
\label{app:converse-proof}

\newcommand{\xprv}{X_+}
\newcommand{\xmrv}{X_-}

Consider the set of all joint distributions of random variables $(X_+, X_-)$
where the marginal distribution of $X_+$ is the same as the distribution of
$Z+\mu_d$ where $Z \sim \nd(.)$, and the marginal distribution of $X_-$ is the
same as the distribution of $Z - \mu_d$. In other words, we consider the set of
all couplings of $Z + \mu_d$ and $Z - \mu_d$. 
In
fact, the marginal distribution of $\xprv$  is the same as that of a data sample
conditioned on $Y = +1$, and the marginal distribution of $X_-$ is the same as
that of a data samples conditioned on $Y = -1$.
Fix a maximal coupling  $(\xprv,
\xmrv)$ in this set, which is defined to
be a coupling that maximizes $\pr{X_+ = X_-}$, or equivalently  minimizes
$\pr{\xprv \neq \xmrv}$.\footnote{Note that $(Z+\mu_d, Z - \mu_d)$ where $Z \sim
  \nd(.)$ is probably not the optimal coupling since $\pr{X_+ \neq X_-} = 1$ unless
  $\mu_d = 0$.}
We use such a maximal coupling to design an effective strategy for the
adversary. Note that maximal coupling is intuitively relevant to adversarial
perturbations, since the adversary wants to change the data so that the samples
conditioned on $Y=+1$ and $Y=-1$ ``look almost the same'', so that the classifier
can extract minimal or no information about the true label upon observing adversarially
perturbed samples. 
Given such a maximal coupling, let
\begin{equation*}
  W :=
  \begin{cases}
    \xprv & \text{if } \xprv = \xmrv, \\
    0 & \text{otw}.
  \end{cases}
\end{equation*}
Moreover, let $Y \sim \text{Unif}(\pm 1 )$ be independent from $(X_+, X_-, W)$
and define
\begin{equation*}
  X :=
  \begin{cases}
    X_+ & Y = +1 \\
    X_- & Y = -1.
  \end{cases}
\end{equation*}
It is easy to verify that $(X,Y)$ have the same joint distribution as our true
feature vector-label pair, i.e.\ for $a \in \reals$ we have 
\begin{equation*}
  \pr{X \leq a | Y = +1} = \pr{X_+ \leq a} = \pr{Z + \mu_d \leq a},
\end{equation*}
and
\begin{equation*}
  \pr{X \leq a | Y = -1} = \pr{X_- \leq a} = \pr{Z - \mu_d \leq a}.
\end{equation*}




Keep in mind that the joint distribution of $(X_+, X_-, W, Y, X)$ depend on $\mu_d$ and hence on $d$. However, we do not make such a dependence
explicit to simplify the notation.



Note that by definition, $W$ is a function of $(X_+, X_-)$ and hence is
independent from $Y$. 
This suggests that $W$ can be considered as a good candidate for the
adversary's perturbation, since the adversary would ideally like to perturb the
data in a way that the information about the true label is removed. More
precisely, given the true label $y$ and data samples $(\xd_i: i \in [d])$, we
generate the modified data samples $\vwd = (\wdd_i: i \in [d])$ such that $\wdd_i$ are
conditionally independent conditioned on $y$ and $\xd_i$, and $\wdd_i$ is generated from the
law of $W$ conditioned on $Y=y$ and $X = \xd_i$. As we discussed above, $W$ is
independent from $Y$, hence  the modified samples $\wdd_i$ do not bear any information
about the label $y$, indicating that $\vwd$ is an ideal candidate for
the adversary. However $\snorm{\vwd - \vxd}_0$ might be above the
adversary's budget $k_d$. In order to address this, we define the final
perturbed data vector $\vxpd$ as follows:
\begin{equation}
  \label{eq:lower-bound-vxpd}
  \vxpd =
  \begin{cases}
    \vwd & \text{if } \snorm{\vwd - \vxd}_0  \leq k_d \\
    \vxd & \text{otw.}
  \end{cases}
\end{equation}
This ensures that indeed $\snorm{\vxpd - \vxd}_0 \leq k_d$. In fact, it turns
out that if $k_d \gg \sqrt{d}$, then $\snorm{\vwd - \vxd}_0 \leq k_d$ with high
probability. The following lemma  will be later useful to make this statement
precise. The proof of Lemma~\ref{lem:adv-modification-expectation} below is given at the end of this section.

\begin{lem}
  \label{lem:adv-modification-expectation}
  Assume that the Assumptions~\ref{assumption:Fisher} and~\ref{assumption:zeta-e-d2} are satisfied and $\mu_d = c
  / \sqrt{d}$ for some $c > 0$. Then for any
  $\delta > 0$ we have
  \begin{equation}
    \label{eq:lim-pr-W-not-X-Y-p-zero}
    \lim_{d \rightarrow \infty} \frac{\pr{W \neq X| Y = +1}}{d^{-\frac{1}{2} + \delta}} = 0,
  \end{equation}
  and
  \begin{equation}
    \label{eq:lim-pr-W-not-X-Y-m-zero}
    \lim_{d \rightarrow \infty} \frac{\pr{W \neq X| Y = -1}}{d^{-\frac{1}{2} + \delta}} = 0.
  \end{equation}
\end{lem}


\begin{proof}[Proof of Theorem~\ref{thm:converse}]
  Assume that the adversary employs the above strategy to perturb the input
  samples. In order to obtain a lower bound for the optimal robust error
  $\optroberrd_{\mu_d, \nd(.)}(k_d)$, we consider any classifier $\mC$.
Let $I$ be the indicator of the event $\snorm{\vwd - \vxd}_0 > k_d$.
We 
  assume that the classifier knows adversary's strategy, and also observes $I$.
  This indeed makes the
 classifier stronger and results in a  lower bound for the robust error. 
Note that if $I=0$, we have $\vxpd = \vwd$ is independent from $y$, and no
classifier can do better than a random guess, resulting in an error $1/2$. In
other words,
\begin{equation*}
  \roberrd_{\mu_d, \nd(.)}(\mC, k_d) \geq \frac{1}{2} \pr{I = 0} = \frac{1}{4} \pr{I = 0 | Y = +1} + \frac{1}{4} \pr{I = 0 | Y = -1}.
\end{equation*}
Since this holds for any classifier $\mC$, we have
\begin{equation}
  \label{eq:converse-proof-optloss-arbit-C}
  \optroberrd_{\mu_d, \nd}(k_d) \geq \frac{1}{2} \pr{I = 0} = \frac{1}{4} \pr{I = 0 | Y = +1} + \frac{1}{4} \pr{I = 0 | Y = -1}.
\end{equation}
Let $I_i, 1 \leq i \leq d$ be the indicator that $\wdd_i \neq \xd_i$.
Using the Markov inequality, we have
\begin{align*}
  \pr{I = 1|Y = +1} &= \pr{\sum_{i=1}^d I_i> k_d | Y = +1} \\
                    &\leq \frac{d \pr{W \neq X | Y = +1}}{k_d} \\
                    &\leq \frac{d \pr{W \neq X | Y = +1}}{d^{1/2 + \epsilon}} \\
                    &= \frac{\pr{W \neq X | Y = +1}}{d^{-1/2 + \epsilon}},
\end{align*}
which goes to zero as $d \rightarrow \infty$ due to
Lemma~\ref{lem:adv-modification-expectation}. Equivalently, $\pr{I = 0| Y = +1}
\rightarrow 1$ as $d \rightarrow \infty$. Similarly, $\pr{I = 0| Y = -1}
\rightarrow 1$ as $d \rightarrow \infty$. Using these
in~\eqref{eq:converse-proof-optloss-arbit-C} we realize that
$\liminf_{d \rightarrow \infty} \optroberrd_{\mu_d, \nd}(k_d) \geq 1/2$  which
completes the proof.
\end{proof}


\begin{proof}[Proof of Lemma~\ref{lem:adv-modification-expectation}]
Let $p_+$ and $p_-$  the distribution of $X_+$ and $X_-$, respectively. The
total variation distance between $p_+$ and $p_-$ is defined to be
\begin{equation*}
  \dtv(p_+, p_-) := \sup_B |p_+(B) - p_-(B)|,
\end{equation*}
where the supremum is over all Borel sets in $\reals$. It is well known that
(see, for instance \cite[Lemma~8.1]{boucheron2013concentration})
if $(X_+, X_-)$ is the optimal
coupling that minimizes $\pr{X_+ \neq X_-}$, we have
\begin{equation*}
  \pr{X_+ \neq X_-} = \dtv(p_+, p_-).
\end{equation*}
We have
\begin{equation}
  \label{eq:converse-pr-W-X-Pinsker}
\begin{aligned}
  \pr{W \neq X | Y = +1} &\stackrel{(a)}{=} \pr{W \neq X_+ | Y= +1} \\
                         &\stackrel{(b)}{=} \pr{W \neq X_+} \\
                         &\stackrel{(c)}{\leq} \pr{X_+ \neq X_-} \\
                         &= \dtv(p_+, p_-) \\
  &\stackrel{(d)}{\leq} \sqrt{\frac{1}{2} D(p_+ \Vert p_-)},
\end{aligned}
\end{equation}
where $(a)$ uses the fact that by definition, conditioned on $Y = +1$, we have
$X = X_+$; in $(b)$ we use the fact that $Y$ is independent from $(X_+, X_-)$
and $W$ is a function of $(X_+, X_-)$; in $(c)$ we use the definition of $W$ to
conclude that if $X_+ = X_-$, we have $W = X_+$; and finally $(d)$ uses Pinsker
inequality (see, for instance, \cite[Theorem~4.19]{boucheron2013concentration}) where $D(p_+\Vert p_-)$
is the Kullback–Leibler (KL) divergence between $p_+$ and $p_-$. With an abuse
of notation, we may use $p_+$ and $p_-$ for the densities of $X_+$ and $X_-$,
respectively, so that $q_+(x) = \nd(x-\mu_d)$ and $q_-(x) = \nd(x+\mu_d)$.
Therefore,
\begin{equation}
  \label{eq:d+--int-q}
  \begin{aligned}
    D(q_+ \Vert q_-) &= \int_{-\infty}^{\infty} q_+(x) \log \frac{q_+(x)}{q_-(x)} dx \\
    &= \int_{-\infty}^{\infty} \nd(x-\mu_d) \log \frac{\nd(x-\mu_d)}{\nd(x+\mu_d)} dx \\
    &= \int_{-\infty}^{\infty} \nd(z) \log \frac{\nd(z)}{\nd(z+2\mu_d)} dz
  \end{aligned}
\end{equation}
Writing the Taylor expansion, we get
\begin{equation*}
  \log \nd(z+2\mu_d) = \log \nd(z) + 2\mu_d \frac{d}{dz} \log \nd(z) + 2\mu_d^2 \frac{d^2}{dz^2}  \log \nd(z+\epsilon_z),
\end{equation*}
where $|\epsilon_z| < 2 |\mu_d|$. Using this in~\eqref{eq:d+--int-q}, we get
\begin{equation}
  \label{eq:D+--taylor-substitution}
  \begin{aligned}
    D(q_+ \Vert q_-) &= \underbrace{- 2 \mu_d \int_{-\infty}^{\infty} \nd(z) \frac{d}{dz} \log \nd(z) dz}_{=:T_1} \\
    &\quad \underbrace{-2\mu_d^2 \int_{-\infty}^\infty \nd(z) \frac{d^2}{dz^2} \log \nd(z+\epsilon_z) dz}_{=:T_2}.
  \end{aligned}
\end{equation}
Observe that from Assumption~\ref{assumption:Fisher}, we have 
\begin{equation}
  \label{eq:converse-T1-zero}
  T_1 = -2 \mu_d\int_{-\infty}^\infty \nd'(z) dz = 0.
\end{equation}
Moreover,
\begin{equation*}
  |T_2| \leq \frac{2c^2}{d} \int_{-\infty}^\infty \nd(z) \left| \frac{d^2}{dz^2} \log \nd(z+\epsilon_z) \right| dz.
\end{equation*}
Since $|\epsilon_z| \leq 2 |\mu_d|$ and $\mu_d = c/\sqrt{d} \rightarrow 0$
as $d \rightarrow \infty$, for $d$ large enough we have $|\epsilon_z| < \zeta$
for all $z \in \reals$, where $\zeta$ is the constant in
Assumption~\ref{assumption:zeta-e-d2}.
Thereby, for $d$ large enough, we have
\begin{align*}
  |T_2| &\leq \frac{2c^2}{d} \int_{-\infty}^\infty \nd(z) \sup_{t \in [z-\zeta, z+\zeta]} \left| \frac{d^2}{dt^2} \log \nd(t)\right| dz \\
        &= \frac{2c^2}{d} \evwrt{Z \sim\nd(.)}{\sup_{t \in [Z-\zeta, Z+\zeta]} \left| \frac{d^2}{dt^2} \log \nd(t)\right|} \\
        &\leq \frac{2c^2}{d} \sqrt{\evwrt{Z \sim\nd(.)}{\sup_{t \in [Z-\zeta, Z+\zeta]} \left| \frac{d^2}{dt^2} \log \nd(t)\right|^2}} \\
        &=: \frac{\alpha}{d},
\end{align*}
where $\alpha$ is the resulting constant, which is finite from
Assumption~\ref{assumption:zeta-e-d2}. Using this together
with~\eqref{eq:converse-T1-zero} in~\eqref{eq:D+--taylor-substitution}, we
realize that for $d$ large enough, we have $D(q_+ \Vert q_-) \leq A/d$. Using
this in~\eqref{eq:converse-pr-W-X-Pinsker}, we realize that for $d$ large
enough, we have
\begin{equation*}
  \pr{W \neq X | Y = +1} \leq \sqrt{\frac{\alpha}{2d}},
\end{equation*}
which implies
\begin{equation*}
  \lim_{d \rightarrow \infty} \frac{\pr{W \neq X| Y = +1}}{d^{-\frac{1}{2} + \delta}} = 0.
\end{equation*}
The proof of~\eqref{eq:lim-pr-W-not-X-Y-m-zero} is similar. This completes the proof.
\end{proof}


\section{Proof of Theorem~\ref{thm:main-phase-transition}}
\label{app:main-phase-transition-proof}

Note that if $\limsup_{d \rightarrow \infty} \log_d k_d < 1/2$, there exists
$\epsilon>0$ such that for $d$ large enough, $\log_d k_d < 1/2 - \epsilon$, or
equivalently $k_d < d^{\frac{1}{2} - \epsilon}$. Therefore, the first part of
the theorem follows from Theorem~\ref{thm:upper-bound}. On the other hand, if
$\liminf_{d\rightarrow \infty} k_d > 1/2$, there exists $\epsilon > 0$ such
that for $d$ large enough, $\log_d k_d > 1/2 + \epsilon$, or equivalently, $k_d
> d^{\frac{1}{2} + \epsilon}$. Therefore, the second part of the theorem follows
from Theorem~\ref{thm:converse}.


\section{Proof of Theorem~\ref{thm:exp-poly}}
\label{app:exp-poly}

Here, we prove that all the
assumptions~\ref{assumption:Fisher}--~\ref{assumption:max-log-q-log} are
satisfied for the noise density $\nd(.)$ of the form~\eqref{eq:q-exp-poly}.
Before proving this, we need some lemmas. The proof of
Lemmas~\ref{lem:poly-sup}, \ref{lem:exp-poly-tail}, and \ref{lem:max-d-q-log-tail-bound} below are
given at the end of this section. 

\begin{lem}
  \label{lem:poly-sup}
  Assume that a degree $n$ polynomial $p : \reals \rightarrow \reals$ is given.
  Given $\epsilon > 0$, we define $\tp : \reals \rightarrow \reals$ as
  follows
  \begin{equation*}
    \tp(x) := \sup_{t \in [x-\epsilon,x+\epsilon]} |p(t)|.
  \end{equation*}
  Then, there exists a polynomial $r: \reals \rightarrow \reals$ with degree $n$,
  such that for all $x \in \reals$, we have $\tp(x) \leq r(|x|)$.
\end{lem}

\begin{lem}
  \label{lem:exp-poly-tail}
  Given the noise density $\nd(.)$ as in~\eqref{eq:q-exp-poly}, there exists a
  constant $c_1 > 0$ such that for all $t \geq c_1$, if $Z$ is a random variable
  with law $\nd(.)$, we have 
  \begin{equation*}
    \pr{|Z| \geq t} \leq \frac{2}{n A a_{2n} t^{2n-1}} \exp\left( - \frac{a_{2n}}{2} t^{2n} \right).
  \end{equation*}
\end{lem}

  \begin{lem}
    \label{lem:max-d-q-log-tail-bound}
    Given the noise density $\nd(.)$ as in~\eqref{eq:q-exp-poly}, there exists a
    constant $c_2 > 0$ such that
    \begin{equation*}
      \lim_{d \rightarrow \infty} \pr{\max_{1 \leq i \leq d} |Z_i| > c_2 (\log d)^{\frac{1}{2n}}}  = 0,
    \end{equation*}
    where $(Z_i: i \geq 1)$ are i.i.d.\ random variables with law $\nd(.)$.
  \end{lem}

\begin{proof}[Proof of Theorem~\ref{thm:exp-poly}]
  As in~\eqref{eq:q-exp-poly}, let
  \begin{equation*}
      \nd(z) = \frac{\exp(\psi(z))}{A},
  \end{equation*}
where
\begin{equation*}
  \psi(z) = -a_{2n} z^{2n} + a_{2n-1} z^{2n-1} + \dots a_1 z + a_0,
\end{equation*}
is a polynomial in $z$ with even degree $2n > 0$  such that $a_{2n} > 0$. We
verify each of the four assumptions separately.

\underline{Assumption~\ref{assumption:Fisher}:}
It is straightforward to check that $\nd(z) > 0$ for all $z$ and $ \nd(.)$
is three times continuously differentiable. 
In order to verify~\eqref{eq:int-qp-qpp-zero}, note that
\begin{equation}
\label{eq:qpz-psip-exp}
  \nd'(z) = \psi'(z) \frac{\exp(\psi(z))}{A},
\end{equation}
and
\begin{equation*}
  \int_{-\infty}^\infty |\nd'(z)| dz = \frac{1}{A} \int_{-\infty}^\infty |\psi'(z)| \exp{(\psi(z))} dz < \infty,
\end{equation*}
where the last step follows from the fact that $\psi'(z)$ is a polynomial in $z$
and $\psi(z)$ is a polynomial with even degree and positive leading coefficient.
This implies that
\begin{equation*}
  \int_{-\infty}^\infty \nd'(z) = \lim_{a \rightarrow \infty} \int_{-a}^a \nd'(z) dz = \lim_{a \rightarrow \infty} \nd(a) - \nd(-a) = 0,
\end{equation*}
since $\nd(z) \rightarrow 0$ as $z \rightarrow \infty$ or $z \rightarrow -\infty$.
Furthermore, 
\begin{equation*}
  \nd''(z) = \left( \psi''(z) + (\psi'(z))^2 \right) \frac{\exp(\psi(z))}{A}.
\end{equation*}
Since $(\psi''(z) + (\psi'(z))^2)$ is a polynomial in $z$, similar to the above
we have $\int_{-\infty}^\infty |\nd''(z)| dz < \infty$. Additionally, it is
evident from~\eqref{eq:qpz-psip-exp} that
$\nd'(z) \rightarrow 0$ as $z\rightarrow \infty$ or $z \rightarrow -\infty$.
Therefore, 
we get $\int_{-\infty}^\infty \nd''(z)dz = 0$ similar to the above. This
establishes~\eqref{eq:int-qp-qpp-zero}.

On the other hand, for the family of
densities $\nd(z;\theta) = \nd(z - \theta)$, we have
\begin{equation*}
  \frac{\partial}{\partial \theta} \log \nd(z; \theta) = \frac{\partial}{\partial \theta} \left( \psi(z - \theta) - \log A \right) = - \psi'(z - \theta).
\end{equation*}
Hence, recalling the definition of the Fisher information, we have
\begin{equation*}
  \mI(\theta) := \evwrt{z \sim \nd(z;\theta)}{\left( \frac{\partial}{\partial \theta} \log \nd(z;\theta) \right)^2 } = \int_{-\infty}^\infty (\psi'(z - \theta))^2 \frac{\exp(\psi(z - \theta))}{A} d z < \infty,
\end{equation*}
since $(\psi'(z - \theta))^2$ is a polynomial in $z$. This means that the above
quantity is well defined and finite, and hence the Fisher information
$\mI(\theta)$ is well-defined and finite for all $\theta$.

\underline{Assumption~\ref{assumption:zeta-d3}} Note that $\frac{d^3}{dt^2} \log
\nd(t)$ is a polynomial in $t$, therefore Lemma~\ref{lem:poly-sup} implies that
for $\zeta > 0$, there exists a polynomial $r: \reals \rightarrow \reals$ such
that
\begin{equation*}
  \sup_{t \in [Z-\zeta, Z+\zeta]} \left |\frac{d^3}{dt^3} \log \nd(t) \right | \leq r(|Z|).
\end{equation*}
Therefore, since all the moments of $\nd(.)$ are finite, and $r(.)$ is a
polynomial, the expectation of the left hand side is finite.

\underline{Assumption~\ref{assumption:zeta-e-d2}} Similar to the above case,
since $\sup_{t \in [Z-\zeta, Z+\zeta]} |\frac{d^2}{dt^2} \log \nd(t) |^{2}$ is
bounded by a polynomial and all the finite moments of $\nd(.)$ are finite, the
expectation is indeed finite.

\underline{Assumption~\ref{assumption:max-log-q-log}}  Note that
  \begin{equation*}
    \frac{d}{dz} \log \nd(z) = \psi'(z) = -2n a_{2n} z^{2n-1} + \dots + a_1.
  \end{equation*}
  Therefore, for all $z \in \reals$, 
  \begin{equation*}
    \left| \frac{d}{dz} \log \nd(z) \right| \leq \sum_{i=1}^{2n} |ia_i| |z|^{i-1},
  \end{equation*}
  and for all $(z_i: i \in [d])$, 
  \begin{equation*}
    \max_{1 \leq i \leq d} \left| \frac{d}{dz} \log \nd(z_i) \right| \leq \sum_{i=1}^{2n} |ia_i| \left( \max_{1 \leq i \leq d} |z_i| \right)^{i-1}.
  \end{equation*}
  Note that if $\max_{i \leq i \leq d} |z_i| \leq c_2 (\log d)^{1/2n}$ with
  $c_2$ being the constant from Lemma~\ref{lem:max-d-q-log-tail-bound}, then
  \begin{equation*}
    \max_{1 \leq i \leq d} \left| \frac{d}{dz} \log \nd(z_i) \right| \leq \sum_{i=1}^{2n} |i a_i| (\log d)^{\frac{i-1}{2n}}.
  \end{equation*}
  Observe that there exists a constant $C_4 > 0$ such that for $d$ large enough,
  we have
  \begin{equation*}
    \sum_{i=1}^{2n} |ia_i| (\log d)^{\frac{i-1}{2n}} \leq C_4 (\log d)^{\frac{2n-1}{2n}} = C_4 (\log d)^{1 - 1/2n}.
  \end{equation*}
  Combining this with the above argument, we realize that for $d$ large enough
  \begin{equation*}
    \pr{\max_{1 \leq i \leq d} \left| \frac{d}{dz} \log \nd(z_i) \right| > C_4 (\log d)^{1 - 1/2n}} \leq \pr{\max_{1 \leq i \leq d} |z_i| > c_2(\log d)^{1/2n}},
  \end{equation*}
  which converges to zero as $d \rightarrow \infty$  from
  Lemma~\ref{lem:max-d-q-log-tail-bound}. This means that
  Assumption~\ref{assumption:max-log-q-log} holds with $C_4$ as above and
  $\gamma = 1-1/2n$.
\end{proof}

\begin{proof}[Proof of Lemma~\ref{lem:poly-sup}]
Let $p(x) = a_n x^n + \dots + a_1 x + a_0$. Let $p'(.)$ be the derivative of $p(.)$.
Since $p'(.)$ is a polynomial of degree $n-1$, it has at most $n-1$ real roots.
Consequently, there exist $-\infty = \alpha_0 < \alpha_1 < \alpha_2 < \dots <
\alpha_{m-1} < \alpha_m = \infty$ where $m \leq n$ and $p(.)$ is monotone in
$[\alpha_i, \alpha_{i+1}]$ for $0 \leq i < m$. Let
\begin{equation*}
  A := \bigcup_{i=1}^{m-1} [\alpha_i - \epsilon, \alpha_i+\epsilon].
\end{equation*}
Note that if $x \notin A$, $p(.)$ is monotone in $[x-\epsilon,x+\epsilon]$. Hence,
\begin{equation}
  \label{eq:x-notin-A-ptilde-bound}
  \tp(x) \leq \max\{|p(x-\epsilon)|, |p(x+\epsilon)|\} \leq |p(x-\epsilon)| + |p(x+\epsilon)| \qquad \forall x \notin A.
\end{equation}
Furthermore, let
\begin{equation*}
  B := \bigcup_{i=1}^{m-1} [\alpha_i - 2\epsilon, \alpha_i + 2\epsilon].
\end{equation*}
Note that $B$ is a compact set, and $p(.)$ is continuous. Therefore, we may
define
\begin{equation*}
  \beta := \max_{x \in B} |p(x)|,
\end{equation*}
and $\beta < \infty$. Since for $x \in A$, we have $[x-\epsilon, x+\epsilon] \subset B$,
we may write
\begin{equation}
  \label{eq:x-in-A-ptilde-bound}
  \tp(x) = \sup_{t \in [x-\epsilon,x+\epsilon]} |p(t)| \leq \sup_{t \in B} |p(t)| = \beta \qquad \forall x \in A.
\end{equation}
Combining this with~\eqref{eq:x-notin-A-ptilde-bound}, we realize that for all
$x \in \reals$, we have
\begin{align*}
  \tp(x) &\leq |p(x-\epsilon)| + |p(x+\epsilon)| + \beta \\
         &\leq \beta + \sum_{i=0}^n |a_i|(|x-\epsilon|^i + |x+\epsilon|^i)  \\
         &\leq \beta + \sum_{i=0}^n 2|a_i|(|x|+|\epsilon|)^i \\
  &=: r(|x|),
\end{align*}
where $r(.)$ is a polynomial of degree $n$. This completes the proof.
\end{proof}

  \begin{proof}[Proof of Lemma~\ref{lem:exp-poly-tail}]
Recalling the polynomial form of $\psi(.)$ and the assumption that $a_{2n} > 0$,
we realize that there exists $c_1 > 0$ such that if $z > c_1$, we have
\begin{equation}
\label{eq:psi-z-large-1/2-2}
  -2a_{2n} z^{2n} \leq \psi(z) \leq -\frac{a_{2n}}{2} z^{2n},
\end{equation}
and if $z < -c_1$, we have
\begin{equation}
\label{eq:psi-z-negative-large-2-1/2}
  -\frac{a_{2n}}{2} z^{2n} \leq \psi(z) \leq -2a_{2n} z^{2n}.
\end{equation}
Thereby, if $t \geq c_1$, we have
\begin{equation}
  \label{eq:z-t-inf-bound}
\begin{aligned}
  \pr{Z \geq t} &= \int_{t}^\infty \frac{1}{A} \exp\left( \psi(z) \right) dz \\
             &\leq \frac{1}{A} \int_t^\infty \exp\left( -\frac{a_{2n}}{2} z^{2n} \right) dz \\
             &= \frac{1}{A} \int_t^\infty \frac{na_{2n}z^{2n-1}}{na_{2n}z^{2n-1}} \exp\left( -\frac{a_{2n}}{2} z^{2n} \right) dz \\
             &\leq \frac{1}{nAa_{2n} t^{2n-1}}\int_t^\infty -\frac{d}{d z} \exp\left(  -\frac{a_{2n}}{2} z^{2n}\right) dz \\
  &= \frac{1}{nA a_{2n}t^{2n-1}} \exp\left( -\frac{a_{2n}}{2} t^{2n} \right).
\end{aligned}
\end{equation}
Similarly, using~\eqref{eq:psi-z-negative-large-2-1/2},  for $t \geq c_1$, we  may write
\begin{equation}
  \label{eq:z--inf--t-bound}
\begin{aligned}
  \pr{Z \leq -t} &= \int_{-\infty}^{-t} \frac{1}{A} \exp\left( \psi(z) \right) dz \\
             &\leq \frac{1}{A} \int_{-\infty}^{-t} \exp\left( -2a_{2n}z^{2n} \right) dz \\
             &= \frac{1}{A} \int_{-\infty}^{-t} \frac{4na_{2n}z^{2n-1}}{4na_{2n}z^{2n-1}} \exp\left( -2a_{2n} z^{2n} \right) dz \\
             &\leq \frac{-1}{4nAa_{2n} t^{2n-1}}\int_{-\infty}^{-t} -\frac{d}{d z} \exp\left(  -2a_{2n} z^{2n}\right) dz \\ 
              &= \frac{1}{4nA a_{2n}t^{2n-1}} \exp\left( -2a_{2n} t^{2n} \right)\\
             &\leq \frac{1}{nA a_{2n}t^{2n-1}} \exp\left( -\frac{a_{2n}}{2} t^{2n} \right).
\end{aligned}
\end{equation}
Combining  \eqref{eq:z-t-inf-bound} and \eqref{eq:z--inf--t-bound} and using the
union bound, we arrive at the desired result. 
  \end{proof}

  \begin{proof}[Proof of Lemma~\ref{lem:max-d-q-log-tail-bound}]
    Since $a_{2n} > 0$, we may choose $c_2$ large enough so that
    \begin{equation}
      \label{eq:c2-condiiton-1}
      \frac{a_{2n}}{2} c_2^{2n} > 1.
    \end{equation}
    Using the union bound, we get
    \begin{equation}
      \label{eq:max-zi-union-bound-1}
      \pr{\max_{1 \leq i \leq d} |Z_i| > c_2 (\log d)^{\frac{1}{2n}}} \leq d \pr{|Z| \geq c_2 (\log d)^{\frac{1}{2n}}},
    \end{equation}
    where $Z \sim \nd(.)$. Using Lemma~\ref{lem:exp-poly-tail}, if $d$ is large enough so that $c_2 (\log
    d)^{\frac{1}{2n}} > c_1$, we have
    \begin{equation}
      \label{eq:pr-z-large-c-log2-1-2n-bound}
      \pr{|Z| \geq c_2 (\log d)^{\frac{1}{2n}}} \leq \frac{2}{nA a_{2n} c_2^{2n-1} (\log d)^{\frac{2n-1}{2n}}}\exp \left( -\frac{a_{2n}}{2} c_2^{2n} \log d  \right).
    \end{equation}
    Using~\eqref{eq:pr-z-large-c-log2-1-2n-bound}
    in~\eqref{eq:max-zi-union-bound-1}, we realize that for $d$ large enough, 
    \begin{align*}
      \pr{\max_{1 \leq i \leq d} |Z_i| > c_2 (\log d)^{\frac{1}{2n}}} &\leq  \frac{2d}{nA a_{2n} c_2^{2n-1} (\log d)^{\frac{2n-1}{2n}}}\exp \left( -\frac{a_{2n}}{2} c_2^{2n} \log d  \right) \\
      &= \frac{2}{nA a_{2n} c_2^{2n-1} (\log d)^{\frac{2n-1}{2n}}}\exp \left( -\left[ \frac{a_{2n}}{2} c_2^{2n}-1  \right] \log d  \right),
    \end{align*}
    which goes to zero as $d \rightarrow \infty$ due
    to~\eqref{eq:c2-condiiton-1}. This completes the proof.
  \end{proof}


\fi

\end{document}